\documentclass[twoside]{article}

\usepackage[accepted]{icml2022}

\usepackage{hyperref}
\usepackage{url}
\usepackage[utf8]{inputenc} %
\usepackage[T1]{fontenc}    %
\usepackage{hyperref}       %
\usepackage{url}            %
\usepackage{booktabs}       %
\usepackage{amsfonts}       %
\usepackage{nicefrac}       %
\usepackage{microtype}      %
\usepackage{xcolor}         %
\usepackage{multirow}
\usepackage{amsmath} %
\usepackage{amssymb}  %
\usepackage{graphicx}
\usepackage{epstopdf}
\usepackage{amsmath}
\usepackage{mathtools}
\usepackage{dsfont}
\usepackage{tikz}
\usepackage{siunitx}
\usepackage{hyperref}
\DeclareMathOperator*{\argmin}{argmin}
\DeclareMathOperator*{\argmax}{argmax}   %

\usepackage[font=footnotesize]{caption}
\usepackage{subcaption}
\usepackage[noabbrev,capitalize]{cleveref}
\usepackage{svg}
\usepackage{balance}    %
\usepackage{amsthm}
\usepackage{tcolorbox}
\usepackage{dsfont}
\usepackage{mathrsfs}
\usepackage[title,toc,titletoc,page]{appendix}
\usepackage{amsmath,stmaryrd,graphicx}
\usepackage{mathtools}
\usepackage{placeins}

\newtheoremstyle{boldStyle}%
  {\topsep}%
  {\topsep}%
  {\itshape}%
  {0pt}%
  {\bfseries}%
  {.}%
  { }%
  {\thmname{#1}\thmnumber{ #2}\thmnote{ (#3)}}

\newtheoremstyle{italicStyle}%
  {\topsep}%
  {\topsep}%
  {}%
  {0pt}%
  {\bfseries}%
  {.}%
  { }%
  {\thmname{#1}\thmnumber{ #2}\thmnote{ (#3)}}

\theoremstyle{boldStyle}
\newtheorem{theorem}{Theorem}
\newtheorem{lemma}{Lemma}

\newtheorem{definition}{Definition}

\theoremstyle{italicStyle}

\newcommand{\BR}{\mathbb R}

\newcommand{\E}{\mathop{\mathbb{E}}}

\begin{document}

\twocolumn[
\icmltitle{LyaNet: A Lyapunov Framework for Training Neural ODEs}

\icmlsetsymbol{equal}{*}
\vspace{-10pt}
\begin{icmlauthorlist}
\icmlauthor{Ivan Dario Jimenez Rodriguez}{yyy}
\icmlauthor{Aaron D. Ames}{yyy}
\icmlauthor{Yisong Yue }{yyy,zzz}
\end{icmlauthorlist}
\vspace{-10pt}
\icmlaffiliation{yyy}{Department of Computational Math and Science, California Institute of Technology}
\icmlaffiliation{zzz}{Argo AI}

\icmlcorrespondingauthor{Ivan Dario Jimenez Rodriguez}{ivan.jimenez@caltech.edu}

\icmlkeywords{Neural ODE, Lyapunov, Stability Learning and Control, Robustness}

\vskip 0.3in
]
\printAffiliationsAndNotice{} 
\begin{abstract}
    We propose a method for training ordinary differential equations by using a control-theoretic Lyapunov condition for stability. Our approach, called LyaNet, is based on a novel Lyapunov loss formulation that  encourages the inference dynamics to converge quickly to the correct prediction. %
    Theoretically, we show that  minimizing Lyapunov loss guarantees exponential convergence to the correct solution and enables a novel robustness guarantee.  We also provide practical algorithms, including one that avoids the cost of backpropagating through a solver or using the adjoint method.
    Relative to standard Neural ODE training, we empirically find that LyaNet can offer improved prediction performance, faster convergence of inference dynamics, and improved adversarial robustness. Our code available at \url{https://github.com/ivandariojr/LyapunovLearning}.

\end{abstract}

\setcounter{footnote}{2} 

\section{Introduction}

The use of dynamical systems to define learnable function classes has  gained increasing attention, first sparked by the discovery of an alternative interpretation of ResNets as Ordinary Differential Equations \citep{haber_stable_2017, e_proposal_2017, ruthotto2018deep, lu2020finite}, and further popularized with Neural ODEs \citep{chen_neural_2019}. 
This fundamental insight holds the potential for a plethora of benefits including parameter efficiency, the ability to propagate probability distributions \citep{chen_neural_2019,rozen2021moser,song2020score}, accurate time-series modeling \citep{chen_learning_2021}, among others.

\begin{figure}[t]
    \centering
    \includegraphics[width=.49\linewidth]{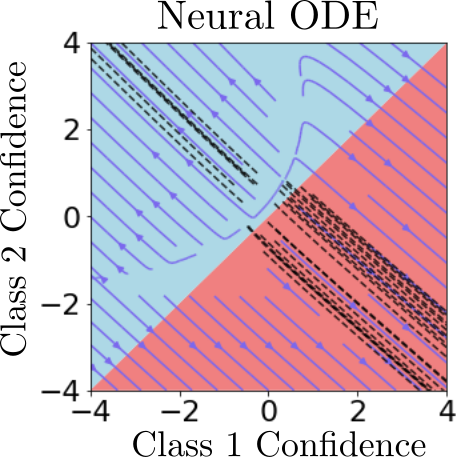}
    \includegraphics[width=.49\linewidth]{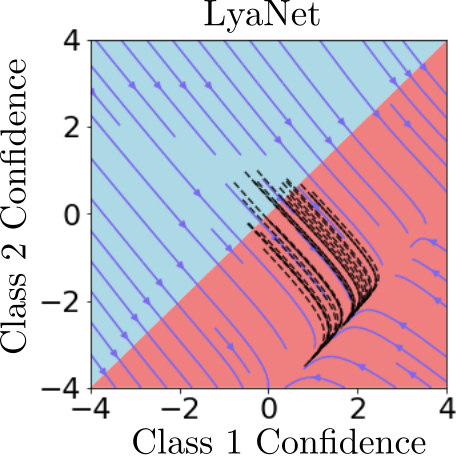} 
    \vspace{-5pt}
    \caption{\textbf{Comparing learned dynamics} on a toy prediction task. The quiver lines show the dynamics (Equation \eqref{eq:node_td}) of a 2-D state space, the dotted black lines show state trajectories from 100 initial conditions (\cref{eq:node_init,eq:node_td}), and the background coloring shows the class label from the output layer (Equation \eqref{eq:node_output_coordiantes}), with red being the correct label.  \textbf{Left:} a Neural ODE does not learn stable dynamics, where some dynamics point towards the incorrect prediction (blue region), and are especially sensitive around the initial conditions.  \textbf{Right:} an identical model trained with LyaNet has much more stable dynamics that always smoothly point towards the correct prediction (red region).
}
    \label{fig:chaos_vs_stabiilty}
\end{figure}

\textbf{Neural ODE Model Class.} Inference or the ``forward pass'' uses a continuous-time ODE, parameterized by $\theta$:  %
\begin{align}
    \label{eq:node_init} \eta(0) &= \phi_\theta(x), \ \ \ \ \  \ \ \ \ \ \mbox{(input layer)} \\
    \label{eq:node_td} \frac{d\eta}{d t} &= f_\theta(\eta, x, t), \ \ \ \mbox{(continuum of hidden layers)}\\
    \label{eq:node_output_coordiantes} \hat{y}(t) &= \psi_\theta(\eta(t)). \ \ \ \ \ \mbox{(output layer)} 
\end{align}
Given input $x$, one makes a prediction (i.e., inference) by solving the ODE specified by Equations \eqref{eq:node_init} \& \eqref{eq:node_td}, and computing the output via Equation \eqref{eq:node_output_coordiantes}.  Without loss of generality, we assume \eqref{eq:node_td} evolves in the time interval $[0, 1]$, i.e., $\hat{y}(t)$ is typically computed at $t=1$, although one could in principle compute $\hat{y}(t)$ at any $t\in[0,1]$.   We also assume that $x\in \BR^n$, $\eta \in H \subset \BR^k$, $\theta \in \Theta \subseteq \mathbb{R}^l$, and $y\in\BR^m$. The state space $H$ is assumed to be bounded and path connected.

\textbf{Motivation.}
The standard learning approach is to differentiate through the ODE solution with techniques such as the adjoint method \citep{chen_neural_2019,kidger2021hey},  which typically does not impose desirable structure, such as stability, within the learned dynamics.\footnote{Some prior work has explored learning dynamical systems that are stable but not explicitly to a correct prediction \cite{ruthotto2018deep,manek_learning_2020,bai_deep_2019,bai_stabilizing_2021}, which can lead to a tension between stability and accuracy that we aim to avoid.}   A such, the inference dynamics can exhibit  unstable behavior that are difficult to integrate, leading to suboptimal prediction performance and solution fragility.
There can also be high computational cost from generating roll-outs during training.

To build some intuition, consider \cref{fig:chaos_vs_stabiilty}, which depicts the flow of two dynamical systems trained to solve a classification task (with almost identical nominal performance). In the left plot, the dynamics from standard Neural ODE training do not reliably converge toward the correct prediction (red region) throughout the state space. This instability leads to slower convergence and incorrect predictions given minor perturbations of the initial conditions (which is related to adversarial robustness). Our goal is to learn a stable dynamical system as shown in the right plot, which reliably and quickly converges to the correct prediction across a robust set of initial conditions.

\textbf{Our Contribution.}
In this work, we study how to train dynamical systems using the control-theoretic principle of Lyapunov
stability.  
Our contributions are as follows:
\vspace{-5pt}
\begin{itemize}
\item We propose LyaNet, a framework for learning dynamical systems with stable inference dynamics.  Our approach is based on a novel Lyapunov loss formulation that captures the degree of violation of the dynamics from satisfying the Lyapunov condition for exponential stability \citep{khalil2002nonlinear,ames2014rapidly}.  
\vspace{-2pt}
\item Theoretically, we show that minimizing Lyapunov loss will: 1) guarantee that the  dynamical system satisfies a Lyapunov condition, and thus during inference or the ``forward pass'' has dynamics that exponentially converge, or stabilize, to the correct prediction; and 2) enable a novel adversarial robustness guarantee inspired by concepts in robust control. %
\vspace{-2pt}
\item We develop practical algorithms for (approximately) optimizing Lyapunov loss, including a Monte Carlo approach that avoids the cost of backpropagating through an ODE solver. Our MC approach exploits the local-to-global structure of the Lyapunov condition where the exponential convergence guarantee arises from every state satisfying a local invariance. 
\vspace{-2pt}
\item We evaluate  LyaNet on a variety of computer vision datasets.  Compared to standard Neural ODE training, we find that LyaNets enjoy competitive or superior prediction accuracy, and  can offer improved adversarial robustness.  We also show that the inference dynamics converge much faster, which implies early inference termination preserves prediction accuracy.
\end{itemize}

\section{Preliminaries}

\subsection{Additional Details on Neural ODEs}
\label{sec:node}

 To have unique solutions for all time, it is sufficient for an ODE to have an initial condition and a globally Lipschitz time derivative as shown in \cref{eq:node_init} and \cref{eq:node_td}, respectively. We thus assume that the dynamics functions we learn are globally uniformly Lipschitz. %
 This assumption is not overly onerous since most neural networks are compositions of globally Lipschitz preserving functions such as ReLUs, Convolutions, max or regular affine functions.

\Cref{eq:node_td} is actually generalization of the original Neural ODE formulation in that $f$ explicitly depends on $x$.  This generalization is sometimes referred to as the Augmented Neural ODE \citep{dupont2019augmented} or Data-controlled Neural ODE \citep{massaroli2020dissecting}. %

\textbf{Connection to ResNet.} One can think of the hidden layers in a ResNet architecture as a discrete-time Euler approximation to  Equation \eqref{eq:node_td} \cite{haber_stable_2017, e_proposal_2017}, where each discretized hidden layer $\eta_t$ is modeled as: 
\begin{eqnarray}
\eta_t = \eta_{t-\delta} + \delta f_{\theta}(\eta_{t-\delta},x,t). \ \ \ \mbox{(ResNet hidden layer)} \label{eqn:resnet}
\end{eqnarray}
It is easy to see that Equation \eqref{eq:node_td} is the continuum limit of Equation \eqref{eqn:resnet} as $\delta\rightarrow0$ (assuming continuity of $f$).

\subsection{Supervised Learning as Inverse Control}
\label{sec:supervised}

We consider the standard supervised learning setup, where we are given a training set of input/output pairs, $(x,y)\sim D$, and the goal is to find a parameterization of our model that minimizes a supervised loss over the training data:
\begin{align}
\argmin_{\theta \in \Theta} \quad & \sum_{(x,y) \sim D}\mathcal{L}(\hat{y}_x(1), y), 
\label{eq:supervised}
\end{align}
where $\hat{y}_x(1)$ is shorthand for  \cref{eq:node_init,eq:node_td,eq:node_output_coordiantes}.

As is typical in deep learning, the standard approach for training Neural ODEs is via backpropagation through  \cref{eq:supervised}.   
The ``end-to-end'' training optimization problem is equivalent to the following finite-time optimal control problem (using just a single $(x,y)$ for brevity):
\begin{align}
    \argmin_{\theta} \quad & \mathcal{L}(\hat{y}(1), y), \label{eq:ftoptc}  \\
    \text{s.t. }& \frac{\partial \eta}{\partial t} = f_\theta(\eta(t), x, t), \nonumber\\
    & \eta(0) = \phi_\theta(x), \nonumber \\
    & \hat{y}(1) = \psi_\theta(\eta(1)). \nonumber
\end{align}
Differentiating $\theta$ through \eqref{eq:ftoptc} requires computing $\frac{\partial \eta}{\partial \theta}$, which was shown to be possible from rollouts of the dynamics using either backpropagation through the solver or the adjoint method \cite{chen_neural_2019,e_proposal_2017}.

\textbf{Challenges.}
In Equation \eqref{eq:ftoptc}, there is no explicit penalty or regularization for intermediate states of the dynamical system. 
As such, even if \eqref{eq:ftoptc} is optimized, the resulting dynamics $\eta(t)$ can exhibit problematic behavior.  Indeed, one can observe such issues in \cref{fig:chaos_vs_stabiilty} where the Neural ODE dynamics are unstable, and in \cref{fig:low_dimensional} where the Neural ODE learns a fragile solution. 
Furthermore, training in this fashion can have high computational costs from generating roll-outs and the inherent numerical difficulty of integrating unstable dynamics.  Our LyaNet approach addresses these limitations via a control-theoretic learning objective.

\subsection{Lyapunov Conditions for Stability}
\label{sec:lyapunov}

In control theory, a stable dynamical system implies that all solutions in some region around an equilibrium point flow to that point. Lyapunov theory generalizes this concept by  reasoning about convergence to states that minimize a potential function $V$.
These potential functions are a special case of dynamic projections \cite{taylor2019pss}.

\begin{definition}[Dynamic Projection \cite{taylor2019pss}]
\label{def:dynproj}
A continuously differentiable function $V: H \rightarrow \BR$ is a dynamic projection if there exist constants $\underline{\sigma}, \overline{\sigma} > 0$ and an $\eta^*$ in $H$ satisfying:
\begin{align}
   \forall \eta \in H:\ \  \underline{\sigma} \lVert \eta-\eta^* \rVert \leq V(\eta) \leq \overline{\sigma} \lVert \eta-\eta^* \rVert .
    \label{eq:dynproj}
\end{align}
\end{definition}

We can now define exponential stability with respect to $V$.  As is common in many nonlinear convergence analyses, analyzing the behavior of a potential function is typically much easier than directly reasoning about the full dynamics.

\begin{definition}[Exponential Stability]
    \label{def:exp_stab}
    We say that the ODE defined in \cref{eq:node_init,eq:node_td} is   \textbf{exponentially stable} if there exists a positive definite dynamic projection potential function $V$ and a constant $\kappa>0$, such that all solution trajectories $\eta(t)$ of the ODE  for all $t \in [0,1]$ satisfy: 
    \begin{align}
        V(\eta(t)) \leq V(\eta(0)) e^{-\kappa t}.
    \end{align}
\end{definition}

Exponential stability implies that the dynamics converge exponentially fast to states with minimal $V$. Later, we will instantiate $V$ using supervised loss.\footnote{For instance, in \cref{fig:chaos_vs_stabiilty}, we want $V\approx0$ only within the red region.  Our definition of $V$ will also involve \cref{eq:node_output_coordiantes}.}  Exponential stability is desirable because: 1) it guarantees fast convergence to desired states (as defined by $V$) after integrating for finite time (e.g., for $t\in[0,1]$); and 2) it has implications for adversarial robustness, discussed later.

In order to guarantee exponential stability, we will impose additional structure, as described in the following theorem.%

\begin{theorem}[Exponentially Stabilizing Control Lyapunov Function (ES-CLF) Implies Exponential Stability (\citet{ames2014rapidly})]
\label{thm:es-clf}
For the ODE in \cref{eq:node_init,eq:node_td}, a continuously differentiable dynamic projection $V$ is an ES-CLF if there is a constant $\kappa > 0$ such that:
\begin{align}
    \min_{\theta \in \Theta} \left[ \left. \frac{\partial V}{\partial \eta }\right|_{\eta} f_\theta(\eta, x, t) + \kappa V(\eta) \right] \leq 0 \label{eq:lya_cond}
\end{align}
holds for all $\eta \in H$ and $t \in [0, 1]$. The existence of an ES-CLF implies that there is a $\theta \in \Theta$ that can achieve:
\begin{align}
 \left.\frac{\partial V}{\partial \eta }\right|_{\eta} f_\theta(\eta, x, t) + \kappa V(\eta) \leq 0,
  \label{eq:lyapunov_achieved_bound}
\end{align}
and furthermore the ODE using $\theta$ is exponentially stable with respect to $V$ (and constant $\kappa$).
\label{thm:lyapunov}
\end{theorem}

\textbf{Local-to-Global Contraction Structure.}
\cref{eq:lyapunov_achieved_bound} is essentially a contraction condition on $V$ with respect to time (with $\kappa$ controlling the rate of contraction). One can further interpret \cref{eq:lyapunov_achieved_bound} as a local invariance property: the condition only depends on the local state $\eta$ rather than, say, the entire trajectory.  In essence, Lyapunov theory exploits this local-to-global structure so that establishing a local contraction-based invariance implies global stability.

\textbf{Control Lyapunov Functions.}
The potential function used in \cref{thm:lyapunov} is called  a control Lyapunov function (CLF).  The main difference between CLFs and classic Lyapunov functions is the additional minimization over $\theta$.  One can think of the parameters $\theta$ as a ``controller'' and \cref{thm:lyapunov} establishes conditions when there exists such a controller that can render the dynamics exponentially stable.\footnote{Conventional applications of \cref{thm:es-clf} focus on designing controllers to stabilize a given physical system \cite{ames2014rapidly}.}  An exponentially stabilizing CLF (ES-CLF) is a CLF where $\kappa$ in \cref{eq:lya_cond,eq:lyapunov_achieved_bound} is strictly greater than zero.

\textbf{Connection to Learnability.}
The minimization in \cref{eq:lya_cond} can be interpreted as a statement about learnability or realizability. Satisfying \cref{eq:lya_cond} equates to the family of parameters $\Theta$ realizing exponential stability with respect to $V$.  A natural way to prove that a $V$ is an ES-CLF is to find (i.e., learn) a parameter $\theta$ that satisfies \cref{eq:lyapunov_achieved_bound}.

\textbf{Learning Exponentially Stable Systems.}
In this paper, we aim to shape a (highly overparameterized) system to satisfy a Lyapunov condition for stabilizing to predictions with minimal supervised loss.
Prior work has explored learning dynamical systems that are stable but not explicitly to a correct prediction \cite{ruthotto2018deep,manek_learning_2020,bai_deep_2019,bai_stabilizing_2021}, which can lead to a tension between stability and accuracy.
Compared to conventional work in Lyapunov analysis, our goal can be viewed as the dual of the more common goal of finding a Lyapunov function to associate with a pre-specified system.

\begin{figure*}[t]
    \centering
    \includegraphics[width=1\textwidth]{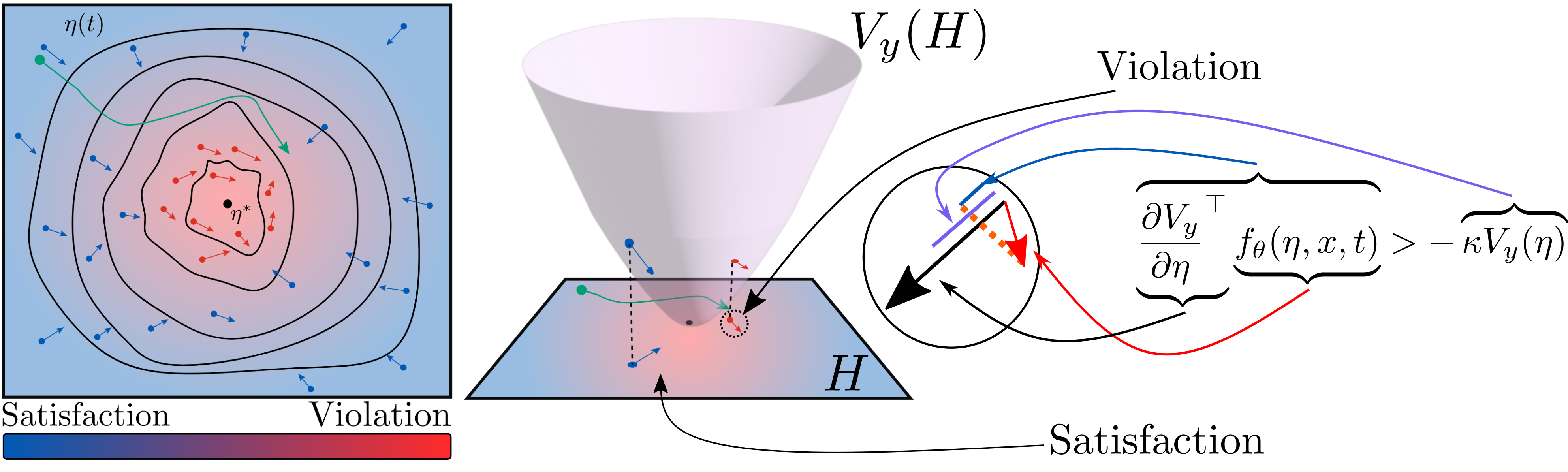}
    \vspace{-20pt}
    \caption{\textbf{Left:} a phase space plot a dynamical system  (\cref{eq:node_td}), along with level sets of the potential function $V_y$ (\cref{eq:V}) which in this example is minimized at $\eta^*$. The blue arrows represent flows that locally satisfy the Lyapunov exponential stability condition \cref{eq:lyapunov_achieved_bound}), while red arrows violate it.  The background coloring indicates a local measure of this violation, as captured in the point-wise Lyapunov loss (\cref{eq:ptwise_lya_loss}). The green line denotes an example trajectory $\eta(t)$ (\cref{eq:node_td}), which in this case does not stabilize to $\eta^*$ which has minimal $V_y$.
    \textbf{Right:} depicting the geometric correspondence between the point-wise Lyapunov loss  and the 1-D projected dynamics of $V_y$, where the inequality is a re-arrangement of terms.   At any state $\eta$, if the point-wise Lyapunov loss is positive (i.e., the depicted inequality is satisfied) then the 1-D projected dynamics at $V_y(\eta)$ is not guaranteed to be exponentially stabilizing to 0.  Conversely, choosing a $\theta$ to break the depicted inequality (and thus achieve zero Lyapunov loss) will guarantee exponential stability.
    }
    \vspace{-10pt}
    \label{fig:lyapunov_convergence}
\end{figure*}

\section{LyaNet Framework}
\label{sec:method}

We now present LyaNet, our Lyapunov framework for training ODEs of the form specified by \cref{eq:node_init,eq:node_td,eq:node_output_coordiantes}.  As alluded to in \cref{sec:lyapunov}, our goal is to find a parameterization $\theta$ of the ODE that satisfies the Lyapunov exponential stability condition in  \cref{thm:es-clf} with respect to a potential function $V$. We develop the formulation in two steps:
\vspace{-5pt}
\begin{enumerate}
    \item 
    \cref{sec:V}: For a given supervised loss, we define an appropriate  potential function $V$.  
    \item \cref{sec:lya_loss}: We define the Lyapunov loss which captures the degree of violation from satisfying the condition in \cref{eq:lyapunov_achieved_bound} that implies exponential stability. %
\end{enumerate}
\vspace{-5pt}
Theoretically, we show that optimizing Lyapunov loss  implies the learned ODE exponentially stabilizes to predictions with minimal supervised loss (\cref{thm:consistency}), which in turn implies a novel adversarial robustness guarantee (\cref{thm:adv_robust}).\footnote{Our results assume the function class is expressive enough to minimize the Lyapunov loss, which we expect from overparameterized dynamical systems.}  We present practical learning algorithms in \cref{sec:algs}.

\subsection{Potential Function for Supervised Loss}
\label{sec:V}

For a given input-output pair $(x,y)$ and supervised loss $\mathcal{L}$, define the following potential function:
\begin{align}
V_y(\cdot ) & := \mathcal{L}(\psi_\theta(\cdot), y), \label{eq:V}
\end{align} 
where the input to $V_y$ depends on $x$, and $\psi_\theta$ is from \cref{eq:node_output_coordiantes}.  Typically, one would input the hidden state at some time $t$, $\eta_\theta(t)$, where the $\theta$ subscript denotes the parameters in \cref{eq:node_init,eq:node_td}.
An additional technical requirement is the continuous-differentiability of $\psi_\theta$, which is satisfied by typical neural network architectures. 

\textbf{Implications.} It is easy to see that $V$ in \cref{eq:V} is minimized when the dynamical system converges to states $\eta$ with zero supervised loss.\footnote{We later show in \cref{fig:ctn_lya_convergence} the convergence rate of inference, which we can exactly interpret as $V$ for cross entropy loss.}  In the following, we further develop the formulation so that we can interpret $V$ as an exponentially stabilizing control Lyapunov function, and thus invoke \cref{thm:es-clf} to prove exponential stability of the inference dynamics to a loss-minimizing prediction.  

\textbf{Truncated Cross Entropy Loss.} In this paper, we focus on the standard cross entropy loss that is widely used for classification.  As a theoretical technical detail, our formulation and analysis use the truncated cross entropy loss defined as: $\mathcal{L}(\cdot):=\max\{0,\mathcal{L}_{ce}(\cdot)-\gamma\}$, where $\mathcal{L}_{ce}$ is cross entropy, and $\gamma>0$ is a small constant. The main difference is that the truncated version attains $\mathcal{L}=0$ at a finite point (whereas cross entropy only attains 0 at infinity), which is required for the stability analysis.  In practice, we can choose $\gamma$ to be machine precision, and ignore it in the learning algorithms. 
\textbf{Dynamic Projections.}
Following the arguments in \cref{sec:lyapunov}, the first step is showing that $V$ in \cref{eq:V} is a dynamic projection, which is analyzed in the following for truncated cross entropy. See \cref{apx:proof_lyaxtropy} for proof.

\begin{lemma}
For the dynamical system described in \cref{eq:node_init,eq:node_td,eq:node_output_coordiantes}, $V_y(\eta)$  in \cref{eq:V} with $\mathcal{L}$ as truncated cross entropy loss is a dynamic projection.
\label{lma:lyaxtropy}
\end{lemma}

\subsection{Lyapunov Loss}
\label{sec:lya_loss}

In order to use the potential function $V$ from \cref{sec:V} to impose stability on ODE training, the remaining step is to satisfy the condition in \cref{eq:lyapunov_achieved_bound}.
To do so, we define the Lyapunov loss, starting with the point-wise version.

\begin{definition}[Point-wise Lyapunov Loss]
\label{def:lyaloss0}
For the dynamical system defined in
\cref{eq:node_init,eq:node_td,eq:node_output_coordiantes}, a single input-output pair $(x,y)$, and dynamic projection $V_y: H \rightarrow R_{\geq 0}$ from \cref{eq:V}, the point-wise Lyapunov loss is:

\begin{small}
\vspace{-15pt}
\begin{align}
    \mathscr{V}(x, y, \eta, t) \coloneqq 
    \max\left\{0,  \frac{\partial V_y}{\partial\eta}^\top f_\theta(\eta,x,t) + \kappa V_y(\eta) \right\}.
    \label{eq:ptwise_lya_loss}
\end{align}
\end{small}
\end{definition} 

Note that the point-wise Lyapunov loss  $\mathscr{V}$ is exactly the violation of the local invariance property in \Cref{eq:lyapunov_achieved_bound}.  Intuitively, for an input-output $(x,y)$, if $\mathscr{V}$ is zero for all $\eta\in H$ and $t\in[0,1]$, then \cref{eq:lya_cond} holds everywhere, and thus \cref{thm:es-clf} implies that the inference dynamics converge exponentially to a loss minimizing prediction.  \cref{fig:lyapunov_convergence} provides a depiction of this intuition, and \cref{thm:consistency} formalizes it.
The Lyapunov loss then applies the point-wise Lyapunov loss to all training examples and possible states.

\begin{definition}[Lyapunov Loss]
\label{def:lyaloss}
For the dynamical system defined in
\cref{eq:node_init,eq:node_td,eq:node_output_coordiantes}, a dataset of input-output pairs $(x,y) \sim D$,  dynamic projection $V_y: H \rightarrow R_{\geq 0}$, and $\mathscr{V}$ from \cref{eq:ptwise_lya_loss}, the Lyapunov loss is:
    \begin{align}
    \mathscr{L}(\theta) \coloneqq \E_{\substack{(x,y) \sim D}}  \left[\int_0^1 \mathscr{V}(x,y, \eta_\theta(\tau), \tau) d\mu(\tau) \right].
    \label{eq:lya_loss}
    \end{align}
\end{definition}

\begin{theorem}
\label{thm:consistency}
Consider the setting in \cref{def:lyaloss}.  If there exists a parameter $\theta^* \in \Theta$ of the dynamical system that attains $\mathscr{L}(\theta^*) = 0$, then for each $(x,y)\sim D$:
\vspace{-5pt}
\begin{enumerate}
    \item The potential function $V_y(\cdot)$ in \cref{eq:V} is an exponentially stabilizing control Lyapunov function with $\theta^*$ satisfying \cref{eq:lyapunov_achieved_bound}.
    \item For $ t \in [0,1]$, the inference dynamics satisfy the following convergence rate w.r.t. the supervised loss $\mathcal{L}$:
    \begin{align}
        \mathcal{L}(\hat{y}_\theta(t), y) \leq \mathcal{L}(\hat{y}_\theta(0), y)e^{- \kappa t},
        \label{eq:pred_convergence}
    \end{align}
    where $\hat{y}_{\theta}$ the output of \cref{eq:node_output_coordiantes} with subscript $\theta$ denoting all parameters in  \cref{eq:node_init,eq:node_td,eq:node_output_coordiantes}.
\end{enumerate}
\vspace{-5pt}
\end{theorem}

\cref{thm:consistency} is essentially a consistency result between minimizing Lyapunov loss and guaranteeing exponential stability on the training examples. See \cref{apx:proof_consistency} for proof.  Future directions include analyzing when the Lyapunov loss is only approximately minimized, as well as generalization.

\textbf{Choosing $\kappa$.} In essence $\kappa$ corresponds to the convergence rate of the loss dynamics. A larger $\kappa$ enables faster convergence.  However, as will come up in Section \ref{sec:algs}, larger $\kappa$ can also make the learning problem more challenging, since the dynamics will have larger magnitude.

\subsection{Adversarial Robustness}
\label{sec:robustness}

One attractive aspect of connecting learnable ODEs to control theory is the ability to leverage concepts in robust control. 
Essentially, robust control provides guarantees that a system will remain stable under perturbations, which is the same kind of guarantee studied in adversarially robust machine learning \cite{wong2018provable,raghunathan2018certified,cohen2019certified,robey2021adversarial}.  One interesting contrast with prior work on adversarially robust learning is that LyaNet does not explicitly optimize for adversarial robustness, but rather the robustness guarantee presented here directly follows from optimizing Lyapunov loss.  %

\begin{definition}[$\delta$-Stable Inference Dynamics for $(x, y)$]
A dynamical system defined in \cref{eq:node_init,eq:node_td,eq:node_output_coordiantes}, with global Lipschitz constant $L$ on $\eta(t)$ and associated
dynamic projections $V_y$ (\cref{eq:V}) has $\delta$-stable inference dynamics for example $(x,y)$ if it satisfies the following conditions:
\vspace{-5pt}
\begin{enumerate}
    \item \textbf{Correct Classification:} 
         $   \argmax_{i  \in \{1 \ldots m\} } \psi_\theta(\eta(1))_i$ $= \argmax_{i  \in \{1 \ldots m\} } y_i$.
    \item \textbf{Exponential Stability: }
        $V_y(\eta(t))$ satisfies the condition in \cref{eq:lyapunov_achieved_bound}.
    \item \textbf{$\delta$-Final Loss: }
          $  V_y(\eta(1)) \leq e^{-\kappa}V_y(\eta(0)) \leq \delta$.
\end{enumerate}
\vspace{-5pt}
\label{def:delta_class}
\end{definition}

\cref{def:delta_class} captures the general properties needed to analyze adversarial robustness. First, the learned model must correctly classify the unperturbed example (implied by minimizing $V_y$). Second, the learned dynamics must be exponentially stable w.r.t. $V_y$ (implied by minimizing Lyapunov loss). Third, at termination of inference ($t=1$), the classification loss encoded in $V_y$ is within an additive constant $\delta\geq0$ from optimal (implied by exponential stability).

\begin{theorem}[Adversarial Robustness of LyaNet]

Consider an ODE defined in \cref{eq:node_init,eq:node_td,eq:node_output_coordiantes} that has $\delta$-stable inference dynamics for input-output  pair $(x,y)$. 
Then the system given a perturbed input $x+\epsilon$ with $\lVert \epsilon \rVert_\infty \leq \overline{\epsilon}$ will produce a correct classification (\cref{def:delta_class}) so long as:
 \begin{align}
     \delta  \leq \log(2) - \frac{L \overline{\epsilon}}{\kappa}\left( 1 - e^{-\kappa}\right).
 \end{align}
 \label{thm:adv_robust}
\end{theorem}
\vspace{-10pt}
See \cref{apx:adv_robust_proof} for proof. Essentially, if a system is exponentially stable, then under perturbation, the system will be exponentially stable to a relaxed set of states (the radius of the relaxation is proportional to the  perturbation magnitude).  So long as that relaxation is contained in the part of the state space that outputs the correct classification, the final prediction is also adversarially robust. Another consequence of \cref{thm:adv_robust} is that the  guarantee is stronger the more accurate the learned model is, since $\delta$ is exactly the nominal cross entropy loss of the prediction.  %

\section{Learning Algorithms}
\label{sec:algs}

We now present two algorithms for (approximately) optimizing Lyapunov loss (\cref{eq:lya_loss}).
The first approach is based on Monte Carlo sampling, and is suitable when the dimension of $\eta$ is small-to-moderate (e.g., tens to hundreds).  The second approach is based on discretized path integrals, and is suitable when the dimension of $\eta$ is very large (e.g., hundreds of thousands).  The main benefit of the MC approach is that it is easily parallelized (since it does not require solving an ODE during training), but may require too many samples to be practical in high dimensions.

The difficulty in optimizing Lyapunov loss is that the distributions for $\eta$ and $t$ are coupled, due to the inner integral in \cref{eq:lya_loss}. If one can effectively sample from this joint $(\eta,t)$ distribution, then one can minimize the (expected) Lyapunov loss by optimizing the parameters $\theta$ w.r.t. the point-wise Lyapunov loss (\cref{eq:ptwise_lya_loss}) at those samples.

\textbf{Restriction on Initial State.}
To simplify algorithm design, we restrict to learning ODEs that always initialize at $\eta_0=0$, i.e., \cref{eq:node_init} is a constant function.  The trained ODEs still perform well in practice. %

\textbf{Monte Carlo Training (\cref{alg:mc_lyanet}).}
The key idea is to sample $\eta$ and $t$ independently from  measures $\mu_H$ and $\mu_{[0,1]}$, which is efficient so long as sampling from $\mu_H$ and $\mu_{[0,1]}$ are efficient. In practice, we choose uniform for $\mu_{[0,1]}$ and discuss $\mu_H$ next.
The resulting (implicit) learning objective upper bounds the Lyapunov loss (\cref{apx:algs}).

One subtlety in defining $\mu_H$ is choosing the support set $H$ of the state space.  Earlier, $H$ had been used solely as a theoretical object, but now must be made explicit.  The key requirement is that  $H$ covers the actual states visited by the ODE during inference.  
Since the ODE function class is globally Lipschitz, we can bound how far states can evolve from the origin.  In practice, we choose $\mu_H$ to be a either a uniform on a hypercube or on hypersphere biased towards the origin, with radius of $H$ being the hypercube with corners at the states that achieve losses of  $e^{\pm \kappa}$ (\cref{apx:algs}).

\begin{algorithm}[t]
  \caption{Monte Carlo LyaNet Training}\label{alg:mc_lyanet}
  \begin{small}
  \begin{algorithmic}[1]
    \State \textbf{hyperparameters:} 
    \State \quad\quad $\mu_H$, $\mu_{[0,1]}$ \Comment{Measures on $\eta$ and $t$}
    \State \quad\quad $\Gamma$ \Comment{Number of MC samples}
    \State \quad\quad $\alpha$, $M$ \Comment{Learning rate and max iterations}
    \State \textbf{initialize} $\theta$ \Comment{Any standard NN initialization}
    \For{$i=1:M$}       
    \State $(x,y) \sim \mathcal{D}$ \Comment{Sample training data}
    \State $\{(\eta_j,t_j)\}_{j=1}^\Gamma \sim \mu_H \times \mu_{[0,1]}$  \Comment{Sample $\Gamma$  $(\eta,t)$ pairs}
    \State $\theta \gets \theta - \alpha \nabla_\theta \sum_{j} \mathscr{V}(x,y,\eta_j,t_j)$ \Comment{From \cref{eq:ptwise_lya_loss}}
    \EndFor
    \State \Return $\theta$
  \end{algorithmic}
  \end{small}
\end{algorithm}

Due to the parallel nature of random sampling, this method often outperforms conventional  Neural ODE training w.r.t.  compute time, since it avoids the sequential integration steps needed for standard backpropagation.\footnote{The computational gains remain even after significantly reducing the precision of the ODE solver.}
For instance, we found that even for 100-dimensional state spaces, only  $\Gamma=500$ samples were required to reliably approximate the integral, which can be very efficient with modern GPUs.
\textbf{Path Integral Training (\cref{alg:pi_lyanet}).}
For systems with extremely high dimensional state spaces (e.g., hundreds of thousands), or for systems that are better approximated in discrete time, we consider an alternative approach based on collecting roll-outs of the original dynamics.

The basic idea is to approximate the inner integral in Lyapunov loss using Euler integration, which uniformly discretizes the integration into $\Gamma$ segments.  We can also define a discrete-time version of the point-wise Lyapunov loss:\footnote{There also exist analogous results to \cref{thm:es-clf} for discrete-time systems \cite{zhang2009exponential}, a thorough treatment of which is beyond the scope of this work.}

\begin{small}
\vspace{-15pt}
\begin{align}
    \mathtt{V}(x,y,\eta, t_j, t_{j-1}) = 
     \max\left\{0,V_y(\eta(t_i)) + (\kappa - 1)V_y(\eta(t_{j-1}))\right\}.\label{eq:disc_lyapunov}
    \end{align}
\end{small}
\hspace{-3.3pt}The resulting discrete-time Lyapunov loss  is then:
 \begin{align}\E_{(x, y) \sim \mathcal{D}} \left[ \sum_{i=1}^\Gamma \mathtt{V}(x, y, \eta, t_i, t_{i-1}) \right]. \label{eq:pi_lya_0th}
 \end{align}
Optimizing \cref{eq:pi_lya_0th} has the  advantage of being more computationally efficient than Monte Carlo training in high-dimensional systems such as  ResNet-inspired Continuous-in-Depth architectures \citep{queiruga2020continuousindepth}.  This is due to the Monte Carlo method placing very little density on the states actually traversed by the ODE, thus requiring many samples to learn reliably. Notice, however, that the underlying inference dynamics remain continuous, we simply apply a discrete-time condition to it. Furthermore, the integral in \cref{alg:pi_lyanet} can be replaced with discrete-time dynamics for application in discrete systems.

\begin{algorithm}[t]
  \caption{Path Integral LyaNet Training}\label{alg:pi_lyanet}
  \begin{small}
  \begin{algorithmic}[1]
    \State \textbf{hyperparameters:}
    \State \quad\quad $\Gamma$ \Comment{Time discretization resolution}
    \State \quad\quad $\alpha$, $M$ \Comment{Learning rate and max iterations}
     \State \textbf{initialize} $\theta$ \Comment{Any standard NN initialization}
     \State \textbf{define} $t_0,t_1,\ldots,t_\Gamma$ \Comment{Evenly spaced from 0 to 1}
    \For{$i=1:M$}
    \State $(x,y) \sim \mathcal{D}$ \Comment{Sample training data}
    \State $\eta(t_j) \gets \int_{t_{j-1}}^{t_j}f(\eta(\tau), x, \tau) d\tau + \eta(t_{j-1}) \quad \forall_{1 \leq j \leq \Gamma}$ \Comment{Generate path integral discretization}
    \State $\theta \gets \theta - \alpha \nabla_\theta  \sum_{j} \mathtt{V}(x,y, \eta(t_j), t_j) $ \Comment{See \cref{eq:disc_lyapunov}}
    \EndFor
    \State \Return $\theta$
  \end{algorithmic}
  \end{small}
\end{algorithm}

\begin{table*}[t]
\small
  \begin{minipage}[c]{0.7\linewidth}
    \centering
\scalebox{0.9}{
    \centering
      \begin{tabular}{c|c|ccc}
       Dataset & Method & Nominal & $\ell_\infty (\epsilon=8/255)$ & $\ell_2 (\epsilon=127/255)$ \\
      \hline
      \multirow{5}*{FashionMNIST} & ResNet18 & 8.17 & 84.55 & 33.69\\
      & Continuous & 7.67 & 91.85 & 45.87\\
      & \bf Continuous LyaNet & \bf 7.18 & \bf 93.96 & \bf 46.54\\
      & Neural ODE & 8.66 & 77.9 & 29.85 \\
      & \bf LyaNet & \bf 7.9 & \bf 46.33 & \bf 25.14\\
      \hline
      \multirow{5}*{CIFAR10} & ResNet18 & 19.73 & 76.44 & 40.06\\
      & Continuous & 13.12 & 88.79 & 41.53\\
      & \bf Continuous LyaNet & \bf 12.99 & \bf 91.03 & \bf 39.82 \\
      & Neural ODE & 18.5 & 75.7 & 38.91\\
      & \bf LyaNet & \bf 17.22 & \bf 58.97 & \bf 38.25\\
      \hline
      \multirow{5}*{CIFAR100} & ResNet18 & 41.07 & 91.08 & 83.72 \\
      & Continuous & 35.96 & 97.12 & 80.34\\
      & \bf Continuous LyaNet & \bf 35.81 & \bf 97.77 & \bf 78.68\\
      & Neural ODE & 41.57 & 92.77 & 81.7\\
      & \bf LyaNet & \bf 41.07  & \bf 91.08 & \bf 81.91 
      \end{tabular}      }
  \end{minipage}
  \hfill
  \begin{minipage}[c]{0.3\linewidth}
    \caption{ 
    Test error percentage comparison for networks trained with our approach and the equivalent network trained with other methods. Continuous LyaNet uses the Continuous Net architecture, and LyaNet uses the Neural ODE architecture. Adversarial robustness results are trained with PGD. LyaNet trained models have similar performance to their counterparts trained by back-propagating through solutions. We note that LyaNet trained with the Monte Carlo approach has a significant robustness enhancement in low-dimensional datasets. \vspace{-10pt}}
    \label{fig:preliminary_results}
      \end{minipage}
\end{table*}

\section{Experiments}
\label{sec:experiments}

In our experiments
we address the following  questions:
\vspace{-5pt}
\begin{itemize}
    \item How well do LyaNets perform compared to their baseline counterparts, including under adversarial attacks?
    \item How quickly can LyaNet inference converge, thus allowing early inference termination?
    \item How does the decision boundary of LyaNet compare with that of other methods?
\end{itemize}
\vspace{-5pt}

\subsection{Experiment Setup}

We compare with three  model classes: 
\vspace{-5pt}
\begin{itemize}
    \item ResNet-18 \citep{he_deep_2015}, which can be viewed as an Euler integrated dynamical system.
    \item A Continuous-in-Depth Network, as presented by \citet{queiruga2020continuousindepth}, trained with both the adjoint method (labeled Continuous) and Path-Integral LyaNet (labeled Continuous LyaNet).
    \item A Data-Controlled Neural ODE where ResNet-18 is used to learn a parameterization of the dynamics trained with both the adjoint method (labeled Neural ODE) and Monte Carlo LyaNet (labeled LyaNet). In both cases the number of hidden dimensions correspond to the number of classes.
\end{itemize}
\vspace{-5pt}

We evaluate primarily on three computer vision datasets: FashionMNIST, CIFAR-10 and CIFAR-100.  More details on the training can be found in \cref{apx:exp_details}

\subsection{Benchmark Experiments}

Table \ref{fig:preliminary_results} shows the primary quantitative results for both standard test prediction error as well as prediction error under PGD bounded adversarial attacks.  %
For standard or nominal test error, we see that our LyaNet variants achieve competitive or superior performance compared to their counterparts trained via direct backpropagation.

For our robustness results, LyaNet trained using the Monte Carlo methods exhibits  improved robustness for CIFAR-10 and FashionMNIST, despite not being explicitly trained to handle adversarial attacks. CIFAR-100 has a significantly larger hidden state dimensionality, and this larger hidden state also hurt nominal performance.  These results are consistent with \cref{thm:adv_robust}, which requires a low nominal error ($\delta$) relative to perturbation size ($\epsilon$) to guarantee adversarial robustness.

We finally note that the Path Integral  method used to train the continuous-in-depth network was not able to improve adversarial robustness. 
This may be due to the coarseness of the path integration, which breaks the continuous exponential stability condition needed to guarantee robustness in \cref{thm:adv_robust} (since we only proved exponential stability for continuous-time integration and not discrete-time).

These results suggest that the LyaNet framework offers promising potential to improve the reliability of Neural ODE training, and more gains may be possible with improved training algorithms for optimizing Lyapunov loss.

\begin{figure}[t]
    \centering
    \includegraphics[width=.49\linewidth]{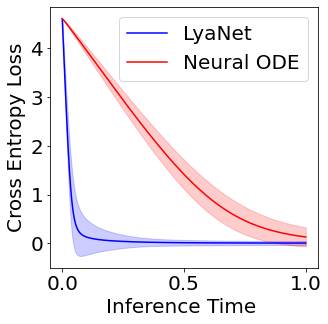}
    \includegraphics[width=0.49\linewidth]{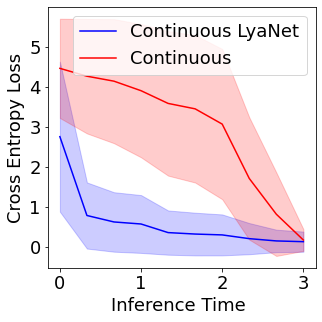}
    \vspace{-5pt}
    \caption{
    Plotting inference time vs prediction loss (if we stopped inference early and made a classification) on 512 correctly classified test examples from CIFAR-100. We see across two model classes that the LyaNet inference dynamics converge much faster.
    }
    \vspace{-10pt}
    \label{fig:ctn_lya_convergence}
\end{figure}

\begin{figure*}[t]
    \centering
    \includegraphics[width=1.15\linewidth]{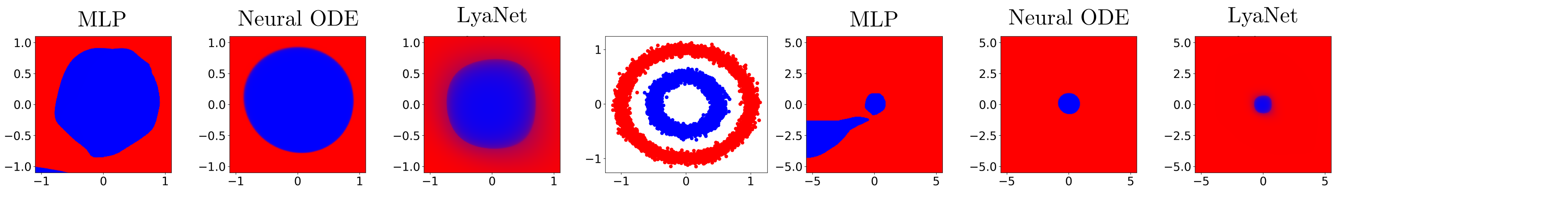}
    \vspace{-20pt}
    \caption{\textbf{Comparing softmax outputs} of learned dynamical systems.  \textbf{Center:} depicting the training data with two classes (red and blue) in a  2-D input space. \textbf{Left 3 Plots:} comparing the softmax outputs, where we see that MPL and Neural ODE have sharp transitions between the two classes, while LyaNet has much smoother transitions.  \textbf{Right 3 Plots:} showing a zoomed out version of left 3 plots.
    \vspace{-10pt}
    }
    \label{fig:low_dimensional}
\end{figure*}

\subsection{Inference Convergence Speed}
Early inference termination refers to stopping the dynamics before $t=1$, and output $\hat{y}(t)$.  %
Intuitively,, one might expect LyaNets to perform well under early inference termination.  

\cref{fig:ctn_lya_convergence} shows the results for the CIFAR-100 dataset.
We see that both the continuous-in-depth models and the data-controlled models show significantly better performance during early termination when trained with LyaNet compared to standard backpropagation.  
These results also demonstrate that the internal dynamics of the learned models are fundamentally different when trained with   LyaNet versus standard backpropagation.

\subsection{Inspecting Decision Boundaries}
To provide some intuition on the the type of decision boundaries being learned, we finally run an experiment on a low-dimensional dataset as shown in \cref{fig:low_dimensional}. All models were trained with the same number of data-points and tested on a uniform grid. The Neural ODE, MLP and (Monte Carlo) LyaNet have the same number of parameters and the models are trained to similar accuracy. We note that because of the low-dimensional setting, this is a best-case scenario for the LyaNet approach. The MLP exhibits the fragility we expect from conventional neural networks. Although the Neural ODE captures the data well, its decision boundary is very sharp and biased towards the outer circle. Meanwhile, the LyaNet is able to capture the uncertainty between the inner and outer circles, and this uncertain boundary can be interpreted as related to adversarial robustness. %

\section{Related Work}
\textbf{Prior Work in Learning Dynamical Systems.}
Most prior work has focused on using the adjoint method to infer dynamics \citep{chen_neural_2019, antonelo_physics-informed_2021}. The proposal by  \citet{e_proposal_2017} even discusses properties like controllability but ultimately frames inference as an optimal control problem.  Although optimal control is a powerful framework, for ODE training it may be challenging  to avoid fragile solutions and attain strong stability guarantees.

Prior work have studied how to impose stability in various forms, including guaranteeing stability around an equilibrium point
\citep{manek_learning_2020,bai_stabilizing_2021,bai_deep_2019,schlaginhaufen2021learning}, and designing architectures that stable by construction \citet{haber_stable_2017}. Drawbacks of these approaches are that they do not guarantee  exponential stability everywhere in the state space, or create a tension between accuracy and stability by using less expressive models.

\textbf{Prior Work in Learning and Control.}
There have been many other intersection points between learning and control theory.
For example, \citet{wilson_lyapunov_2018} use Lyapunov theory to analyze the dynamical system implicit in the momentum updates of stochastic gradient descent, \citet{DBLP:journals/corr/abs-1810-13400, info_theoretic_mpc}  differentiate through controllers like MPC, and \citet{peng2020learning} learn control policies directly for a real system. \citet{richards_lyapunov_2018} safely learn physical dynamics by taking into account Lyapunov-like conditons during training. \citet{chang_neural_2020} use an adversarial approach to learn Lyapunov functions for control.  \citet{cheng2019control} uses a control prior to regularize reinforcement learning, in some cases with robustness properties of the learned policy.  \citet{chow2018lyapunov} uses a Lyapunov condition to improve safety conditions in reinforcement learning.  \citet{dean2020sample} study the sample complexity of learning controllers such as for LQR.  \citet{rosolia2017learning} study how to learn cost-to-go value functions for  MPC under safety constraints. Learning Lyapunov functions for controlling hybrid dynamical systems \citet{chen2021learning}. Learning Lyapunov stability certificate for a black-box neural network policy \citet{richards2018lyapunov}. Neural networks have also been used to learn contraction metrics (a similar concept to Lyapunov) for the control and estimation of stochastic systems \citet{tsukamoto2020neural}. 

\section{Discussion \& Future Work}
\vspace{-5pt}
This work studies a new connection between learning and control in the context of training network architectures built using dynamical systems. In contrast to studying physical systems where the state description has a real physical interpretation that cannot be modified, the internal state of Neural ODEs is highly controllable (or modifiable) due to overparameterization.  As a consequence, the ``control'' problem here can rely on relatively simple control theoretic tools due being able to design (or train) the entire system, rather controlling only a part of the system to force the whole system to be stable (e.g., controlling the motor of a seqway so that it stands upright).  The tight integration of the Lyapunov condition via optimizing Lyapunov loss also automatically confers many benefits such as adversarial robustness.

There some many interesting directions for future work. For instance, our learning algorithms are the natural starting points for optimizing Lyapunov loss, but can be significantly improved.  It may also be interesting to combine LyaNet with other approaches for improving adversarial robustness.  One can also study other concepts in control theory, such as barrier conditions \cite{ames2016control} instead of Lyapunov conditions, to guarantee set invariance rather than stability. Finally, the fast inference dynamics of LyaNet may have applicability in time-critical inference settings.

\FloatBarrier
\bibliographystyle{icml2022}
\bibliography{references}

\onecolumn
\appendix
\section{Proofs}
\subsection{Auxiliary Facts}
\begin{definition}[Dynamic Projection]
A continuously differentiable function $\psi : \BR^k \to \BR^m$ is a dynamic projection if there exist $\underline{\sigma}, \overline{\sigma} \in \BR_{>0}$ so that:
\begin{align}
    \underline{\sigma} \lVert \eta(t) \rVert \leq \lVert \psi(\eta(t)) \rVert \leq \overline{\sigma} \lVert \eta(t) \rVert
\end{align}
\label{def:dyn_proj}
\end{definition}

\begin{lemma}[Composition of Dynamic Projections]
The composition of dynamic projections is a dynamic projection. \label{lma:dyn_proj_lemma}
\end{lemma}
\begin{proof}
Let $\psi$ and $V$ both be dynamic projections with corresponding constants $\underline{\sigma}^\psi, \underline{\sigma}^V, \overline{\sigma}^\psi, \overline{\sigma}^V$.
Then $V \circ \psi$ satisfies the following inequalities:

\begin{align}
\underline{\sigma}^V \lVert \psi(\eta(t)) \rVert \leq \lVert V(\psi(\eta(t))) \rVert \leq \overline{\sigma}^V \lVert \psi(\eta(t)) \rVert \\
\underline{\sigma}^\psi \underline{\sigma}^V \lVert \eta(t) \rVert \leq \lVert V(\psi(\eta(t))) \rVert \leq \overline{\sigma}^V \overline{\sigma}^\psi \lVert \eta(t) \rVert \\
\end{align}

Therefore $V \circ \psi$ is a dynamic projection with constants $\underline{\sigma}^\psi \underline{\sigma}^V$ and $\overline{\sigma}^V \overline{\sigma}^\psi $.

\end{proof}

\subsection{Proof of \cref{lma:lyaxtropy}}
\label{apx:proof_lyaxtropy}
\begin{proof}
The cross entropy loss function is the composition of the negative-log-likelihood loss function, softmax operation and the output coordinates $\psi_\theta$. 
Using \cref{lma:dyn_proj_lemma}  all we need to prove is that for the set of admissible times where $\eta(t)$ evolves, the negative log-likelihood and the softmax operation are dynamic projections. 
Note that $\psi_\theta$ is assumed to be a dynamic projection since it is usually a globally Lipschitz neural network.

Since we assume $f$ is globally Lipschitz we know $\eta(t)$ exists and is unique for all time. 
Since $t$ is in the compact set $[0, 1]$, we can conclude that $\mathcal{H} = \{ \eta(t): t \in [0,1]\}$ is also compact.
We define $\eta^* = \argmin_{h \in \mathcal{H}} \mathcal{L}(\psi_\theta(\eta), y)$, in other words the closest point in $\mathcal{H}$ to the correct classification under the cross-entropy loss function.
Since $\mathcal{H}$ is compact we know that supremums and infimums with respect to the objective $\lVert \eta - \eta^* \rVert$ are attained within the set.
Let $\overline{h} = \argmax_{\eta \in \mathcal{H}} \lVert \eta - \eta^* \rVert $ and $\underline{h} = \argmin_{\eta \in \mathcal{H}} \lVert \eta - \eta^* \rVert $.

Therefore, we can conclude that the Softmax operation is a dynamic projection with the following bounds:

\begin{align}
     \frac{\sigma^\psi}{\lVert \overline{h} \rVert }\lVert \eta(t) - \eta^* \rVert \leq \lVert \text{Softmax}(\psi(\eta(t))) \rVert \leq \frac{\sigma^\psi}{\lVert \underline{h} \rVert } \lVert \eta(t) - \eta^* \rVert
\end{align}

Since $\psi(\eta(t))$ is bounded we know that all entries $i$ of the output of the Softmax operation satisfy the following inequality for all $t \in [0, 1]$:

\begin{align}
 0 < \text{Softmax}(\psi(\eta(t)))_{i} < 1   
\end{align}
Let $\underline{s} = \min_{i \in m, t \in [0,1]} \text{Softmax}(\psi(\eta(t)))_{i}$ and $\overline{s} = \max_{i \in m, t \in [0,1]} \text{Softmax}(\psi(\eta(t)))_{i}$. Then the negative log likelihood for one-hot-encoded label $y$, $V_y$ satisfies the dynamic projection constrains:

\begin{align}
    \frac{\sigma^\psi}{-\log(\overline{s})\lVert \overline{h} \rVert }\lVert \eta(t) - \eta^*\rVert \leq V_y(\text{Softmax}(\psi(\eta(t))))  \leq \frac{\sigma^\psi}{-\log(\underline{s})\lVert \underline{h} \rVert } \lVert \eta(t) - \eta^* \rVert
\end{align}

\end{proof}

\subsection{Proof of \cref{thm:consistency}}
\label{apx:proof_consistency}
\cref{thm:adv_robust} is inspired by concepts underlying Input-to-State and Projection-to-State stability \cite{sontag1995characterizations}. 
The key technical detail is leveraging the Comparison Lemma in a similar fashion as  \citet{taylor2019pss}. 
\begin{proof}
    We begin by rearranging the terms of the integral:

\begin{align}
    \mathscr{L}(\theta) = E_{\substack{(x,y) \sim D}}  \left[\int_0^1 \mathscr{V}(x,y, \eta_\theta(\tau), \tau) d\mu(\tau) \right] &= \int_{D}  \int_0^1 \mathscr{V}(x,y, \eta_\theta(\tau), \tau) d\mu(\tau) dD((x,y)) \\ 
    &= \int_{D \times [0,1]}   \mathscr{V}(x,y, \eta_\theta(\tau), \tau) d\mu\times D(\tau, (x,y)) \\ 
\end{align}

Recall that 
    \begin{align}
     \mathscr{V}(x, y, \eta, t) &= \max\left\{0,  \frac{\partial V_y}{\partial\eta}^\top f_\theta(\eta,x,t) + \kappa V_y(\eta) \right\}
    \end{align}
We note that  that $\mathscr{V}$ is being integrated over a bounded domain and it satisfies the following properties:

\begin{enumerate}
    \item $\mathscr{V}(x, y, \eta(t), t) \geq 0$ for all values of $x, y, \eta$ and $t$.
    \item $\mathscr{V}(x, y, \eta(t), t)$ is continuous since it is the maximum of two continuous functions: the $0$ function and  $\frac{\partial V_y}{\partial\eta}^\top f_\theta(\eta,x,t) + \kappa V_y(\eta)$ which is continuous since it is differentiable.
\end{enumerate}

From these two facts we can conclude that $\mathscr{L}(\theta) = 0 $ implies that for all $t, (x,y) \in [0,1] \times D$ the function $\mathscr{V}(x, y, \eta(t), t) = 0$. This follows from the standard calculus argument that if the function weren't zero at a point there would be an $\epsilon$ region surrounding that point that would integrate to a strictly positive value. Since $\mathscr{V} \geq 0 $, integrating it over any region must also be non negative which would be a contradiction given that we assumed the integral is non-negative.
\end{proof}

\subsection{Proof of \cref{thm:adv_robust}}

\cref{thm:adv_robust} is inspired by concepts underlying Input-to-State and Projection-to-State stability \cite{sontag1995characterizations}. The key technical detail is leveraging the Comparison Lemma in a similar fashion as  \citet{taylor2019pss}.

\label{apx:adv_robust_proof}
\begin{proof}
We will use the following notation for the time derivative of $V_y$:
\begin{align}
 \dot{V}_y(\eta(t),x,t) = \frac{d}{dt} V_y(\eta(t)) = \left. \frac{\partial V_y}{\partial \eta}\right|_{\eta = \eta(t)}^\top f(\eta(t), x, t)   
\end{align}

We now begin the derivation.\footnote{Note that the choice of norm is arbitrary given equivalence of norms in finite dimensional spaces}

\begin{align}
     &\dot{V}_y(\eta,x + \epsilon,t) \\
     &= \dot{V}_y(\eta,x,t) + \dot{V}_y(\eta,x + \epsilon,t) - \dot{V}_y(\eta,x,t)&& \text{Add 0} \\
     &\leq \dot{V}_y(\eta,x,t) + \lvert \dot{V}_y(\eta,x + \epsilon,t) - \dot{V}_y(\eta,x,t) \rvert \\
     &\leq \dot{V}_y(\eta,x,t) +  \left\lvert \frac{\partial V_y}{ \partial \eta}^\top f(\eta,x + \epsilon,t) - \frac{\partial V_y}{ \partial \eta}^\top f(\eta,x,t) \right\rvert  \\
     &\leq \dot{V}_y(\eta,x,t) + \underbrace{L_{V}L_{f}}_{L}\lVert \epsilon \rVert && \text{From Global Uniform Lipschitz property of $V$ and $f$} \\
     &\leq \dot{V}_y(\eta,x,t) +  L \overline{\epsilon} && \text{ Adversarial disturbance bound } \lVert \epsilon \rVert \leq \overline{\epsilon}  \\
     &\leq -\kappa V_{y}(\psi(\eta(t)) + L \overline{\epsilon}&& \text{From Lyapunov Exponential Stability condition in \cref{thm:lyapunov}} \label{eq:intermediate_upper_bound}
\end{align}

We now consider the dynamical system that achieves the upper bound \cref{eq:intermediate_upper_bound} in preparation to apply the comparison lemma:

\begin{align}
    \dot{\gamma} &= -\kappa \gamma(t) + L \overline{\epsilon}
\end{align}

This dynamical system is linear and therefore has solutions of the following form:

\begin{align}
    \gamma(t) = e^{-\kappa t} c + \frac{L \overline{\epsilon}}{\kappa}
\end{align}

Now we solve for $c$ in terms of the initial condition of this system:
\begin{align}
 \gamma(0) &= c + \frac{L \overline{\epsilon}}{\kappa}  \\
 &\rightarrow c = \gamma(0)  - \frac{L \overline{\epsilon}}{\kappa}\\
 &\rightarrow \gamma(t) = e^{-\kappa t} \gamma(0) + \frac{L \overline{\epsilon} }{\kappa}(1 - e^{-\kappa t}) 
\end{align}
Where the last line comes from plugging the resulting value of $c$ back into the solution $\gamma(t)$. 

To apply the Comparison Lemma, we need a choice of $\gamma(0)$ that satisfies $V_y(\eta(0)) \leq \gamma(0)$. 
Recall from the $\delta$-Final Loss condition of \cref{def:delta_class} that $e^{-\kappa} V_y(\eta(0)) \leq \delta$. This implies that $V_y(\eta(0)) \leq \delta e^{\kappa}$ so that we can pick $\gamma(0) = \delta e^{\kappa}$ to satisfy the initial condition inequality of the comparison lemma. Plugging this value into our form of the solution for $\gamma(t)$ results in the following:
\begin{align}
    \gamma(t) = e^{\kappa(1-t)} \delta + \frac{L \overline{\epsilon}}{\kappa}(1 - e^{-\kappa t}) 
\end{align}

Therefore, since the following conditions hold:
\begin{enumerate}
    \item $\dot{V_y}(\eta(t))$ and $\dot{\gamma}(\gamma(t))$ are continuous functions of time and state.
    \item $\dot{V_y}(\eta(t)) \leq \dot{\gamma}(\gamma(t))$
    \item $V_y(\eta(0)) \leq \gamma(0)$
\end{enumerate}
We can conclude from the comparison lemma that $V_y(\eta(t)) \leq \gamma(t)$. 

Recall from the definition of $V_y$ as the cross entropy loss, that $V_y(\eta(t)) \leq -\log(\frac{1}{2}) = \log(2)$ is sufficient for to guarantee Correct Classification as defined in \cref{def:delta_class}. Therefore, it suffices to show that $\gamma(t) \leq \log(2)$. Plugging in the definition of $\gamma(t)$ evaluated at the end of the integration time $t=1$ we obtain:

\begin{align}
    \delta + \frac{L \overline{\epsilon}}{\kappa}(1 - e^{-\kappa}) \leq \log(2) \\
    \delta \leq \log(2) - \frac{L \overline{\epsilon}}{\kappa}(1 - e^{-\kappa})
\end{align}

Which results in the desired bound.

\end{proof}

\newpage
\section{Details on Learning Algorithms}
\label{apx:algs}

\textbf{Monte Carlo Method.}
We can formulate the objective of the training optimization problem as a Monte Carlo integral as follows:
\begin{align}
    \min_{\theta \in \Theta} \mathscr{L}(\theta) &\leq \min_{\theta \in \Theta} \E_{\substack{(x, y) \sim \mathcal{D} \\ \eta\sim \mu_H \\ t \sim \mu_{(0,1)}}}\left[ \mathscr{V}(x, y, \eta, t)\right] \\
   &\leq \min_{\theta \in \Theta}  \E_{(x, y) \sim \mathcal{D}}\left[ \int_{H \times (0,1)} \mathscr{V}(x, y, \eta, t) d\mu_{H \times (0,1)} (\eta,t) \right],
\end{align}
which establishes it as an upper bound on Lyapunov loss.

Some of the stated assumptions require justification. 
First, we know that perfect classifications using the cross entropy loss function require infinite inputs in one coordinate.
Exponential stability to a correct classification under the cross entropy lyapunov dynamics would thus imply exponential growth for the hidden state.
To avoid numerical stability issues we bound the state-space where the dynamics can evolve through our choice  $\kappa$ and by noting that most neural network architectures are globally Lipschitz. 
This implies uniqueness and existence of solutions for all time.
When taking into account that we only compute a finite time integral, we can conclude that the image of a bounded set of initial conditions through the dynamics's evolution will remain be bounded. 
Therefore, we only need to enforce a bounded set of initial condition to have a bounded reachable set.
This means our predictions can only have a maximum confidence but with a sufficiently a large bounded region where the dynamics can evolve, this error can become arbitrarily low.
In practice, we initialized the dynamics with the constant zero function.

The choice of measure for both $\eta$ and $t$ has the potential to significantly impact the convergence integral; however, in practice we found that a uniform distribution over the n-cube performed well for most problems and in high-dimensional settings a slightly different distribution can be used to mitigate the curse of dimensionality.

\textbf{Picking $H$ to sample from:}

We will sample from a $k$-dimensional hypercube with corners at $\left( \pm s, \pm s, \ldots , \pm s \right)$. 
We now have to pick $s$ so that $\eta$ does not evolve outside of the set $H$. 
In principle we could instead select a $k$-hypersphere but it is well known that most of the volume of the hypersphere is concentrated towards the surface in higher dimensions. 
In reality, the dynamics will spend most of the time evolving in the interior since our prediction dynamics are initialized at the origin.
Although the hypercube also has most of its volume near the corners, we expect systems to evolve near the corners that correspond to a classification.
We note that this choice of $H$ is unlikely to be optimal.

 To select and $s$, begin by considering the bound we are trying to enforce on the dynamics:

\begin{align}
    \dot{V}_y(\eta, x, t) \leq -k V(\eta)
\end{align}
In a comparison lemma style analysis consider the system that achieves the upper bound:

\begin{align}
    \dot{\gamma} = -k \gamma
\end{align}

We hypothesise that in the case of an incorrect classification, the system evolves as if lower bounded by $\underline{\gamma} = k \underline{\gamma}$.
This assumption is based on the supposition that the dynamical system will evolve to a corner of the hypercube corresponding to an admissible but incorrect classification.
Therefore we only have to solve for an $s$ that satisfies $V_y(\eta^*) = e^{-k}$ where 

\begin{align}
    \eta^*_i = \begin{cases}
    s & \text{ if } 1\leq i \leq k \text{ is the correct class. } \\
    -s & \text{ otherwise }\\
    \end{cases}
\end{align}

Although this choice of $H$ worked in practice, this choice of $H$ is not optimal and could be improved by better exploiting the structure of the dynamics and the classification problem more generally.

\textbf{Sampling on Hypersphere}
In high-dimensional spaces sampling in the hypercube is highly inefficient for retrieving samples around the origin.
Our models are also all initialized at the origin and should end close to a vertex of the simplex that corresponds to a classification.
It's thus imperative to bias sampling towards the origin in higher dimensions. To do this we use the following procedure to generate samples: 

\begin{enumerate}
    \item Set $r_{max}$ so that a hyper sphere of that radius inscribes the hyper-cube for $H$ defined previously.
    \item Sample a radius $ r \sim U(0, r_{max})$
    \item Sample $h$ from the Hyper sphere of radius $r$ of dimension $n-1$  where $n$ is the dimensionality of the state-space using the method by \citet{hypersphere_sampling}.
\end{enumerate}
Although this method significantly biases the samples towards zero, we note that that the region near zero is the area that implies the smallest prior for a classification.
If we are very far away from the origin we would expect to keep moving in the direction opposite of the origin (since we have a strong prior).
The dynamics should learn to exploit this by having relatively simple dynamics in points far away from the origin that get more complex as they approach the point with lowest prior (zero).

\newpage
\section{Experimental Details}
\label{apx:exp_details}
To simplify tuning, we trained our models using Nero \citep{liu2021learning} with a learning rate of $0.01$ with a batch size of $64$ for models trained with LyaNet and $128$ for models trained with regular backpropagation. We found this by performing a grid search on learning rates and batch sizes over $(0.1,0.001,0.001)\times(32,64,128)$, validated on a held out set of $10\%$ of training data. All models were trained for a total of 120 epochs. For our adversarial attack we used PGD as implemented by \citet{kim2020torchattacks} for 10 iterations with a step size $\alpha=\frac{2}{255}$. Our experiments ran on a cluster 6 GPUs: 4 GeForce 1080 GPUs, 1 Titan X and Titan RTX. All experiments were able to run on less than 10GB of VRAM. 

\end{document}